\documentclass[twoside,leqno,twocolumn]{article}
\usepackage{ltexpprt}

\usepackage{tikz}
\usetikzlibrary{bayesnet}
\usepackage{cite}
\usepackage{amsmath,amssymb,amsfonts}
\usepackage{algorithmic}
\usepackage{graphicx}
\usepackage{multirow}
\usepackage{times}
\usepackage{hyperref}

\usepackage{booktabs}

\usepackage{varwidth}

\begin{document}

\title{\Large TNE: A Latent Model for Representation Learning on Networks}
\author{Abdulkadir \c{C}elikkanat\thanks{CentraleSup\'elec, University of Paris-Saclay and Inria Saclay, France Email: $\{$\texttt{firstname.lastname}$\}$\texttt{@centralesupelec.fr}} \\
\and
Fragkiskos D. Malliaros\footnotemark[1]}
\date{}

\maketitle







\begin{abstract} \small\baselineskip=9pt 
Network representation learning (NRL) methods aim to map each vertex into a low dimensional space by preserving the local and global structure of a given network, and in recent years they have received a significant attention thanks to their success in several challenging problems. Although various approaches have been proposed to compute node embeddings, many successful methods benefit from random walks in order to transform a given network into a collection of sequences of nodes and then they target to learn the representation of nodes by predicting the context of each vertex within the sequence. In this paper, we introduce a general framework to enhance the embeddings of nodes acquired by means of the random walk-based approaches. Similar to the notion of \textit{topical word embeddings} in NLP, the proposed method assigns each vertex to a topic with the favor of various statistical models and community detection methods, and then generates the enhanced community representations. We evaluate our method on two downstream tasks: node classification and link prediction. The experimental results demonstrate that the incorporation of vertex and topic embeddings outperform widely-known baseline NRL methods.
\end{abstract}

\section{Introduction}
Graphs are important mathematical structures commonly used to represent the objects and their relations in real-world systems such as the World Wide Web, social networks, and protein-protein interactions. Due to the wide range of applications that networks appear, network analysis methods have attracted great interest from the research community, and numerous techniques have been proposed to better understand and uncover their underlying properties.  In recent years, many prominent and powerful approaches have emerged under the field of \textit{network representation learning} (NRL). The main goal of NRL techniques is to learn feature vectors corresponding to the nodes of the graph (also known as \textit{node embeddings}), by preserving important structural properties of the network; those vectors can later be used to perform various analysis and mining tasks including visualization, node classification and link prediction with the favor of machine learning algorithms.

The initial studies in the field of node representation learning, have mostly relied on \textit{matrix factorization} techniques, since various properties and  interactions between nodes can be expressed as matrix operations. However, these methods are mainly applicable on small-scale networks due to their high computational cost -- especially for graphs consisting of millions of nodes and edges \cite{DBLP:journals/debu/HamiltonYL17}. More recent studies have concentrated on developing methods suitable for relatively large-scale networks --  being able to effectively approximate the underlying objective functions that capture meaningful information about the nodes of the graph and their properties.

A plethora of node representation learning methods have been inspired by the advancements in the area of \textit{natural language processing} (NLP), borrowing various ideas originally developed for computing \textit{word embeddings}. One such successful technique is the \textit{Skip-Gram} architecture \cite{word2vec}, which aims to find latent representations of words by estimating their context in the sentences of a textual corpus. That way, many pioneer studies in NRL utilize the idea of random walks to transform graphs into a collection of sentences -- as an analogy to the area of natural language -- and these sentences or walks are later being used to learn node embeddings. 
 
Although random walk-based approaches are strong enough to capture  local connectivity patterns, they mainly suffer to sufficiently convey information about the global structural properties of the network. More precisely, real-world networks have an inherent clustering (or community) structure, which can be utilized to further improve the predictive capabilities of node embeddings. One can interpret such structural information based on an analogy to the concept of \textit{topics} in a collection of documents. In a similar way as word embeddings can be enhanced with topic-based information \cite{topical_word}, here we aim at empowering node embeddings by employing information about the community structure of the graph --  that can be achieved by a process similar to the one of \textit{topic modeling}.

In this paper, we propose \textit{topical node embeddings} (TNE), a framework in which node and topic embeddings are learned separately from the network, and then they are merged into a single vector -- leading to further improvements in the performance on downstream tasks.  The main contributions of the paper can be summarized as follows:

\begin{itemize}
	\item \textit{A novel node representation learning framework}. We propose a new strategy, called TNE, which learns community embeddings from the graph, and use them to improve the node representations extracted by random walk-based methods.
	
	
	\item \textit{Enriched feature vectors}. We perform a detailed empirical evaluation of the embeddings learned by TNE on the tasks of node classification and link prediction. As the experimental results demonstrate, the proposed model provides  feature vectors which can boost the performance of downstream tasks.

\end{itemize}

The rest of the paper is organized as follows. In Section \ref{sec: related_work}, we describe the related work. In Section \ref{sec: prob_def}, we formulate the problem, and in  Section \ref{sec:model} we present the proposed method. Section \ref{sec:experiments} presents the experimental results, and finally, in Section \ref{sec:conclusions} we conclude our work providing also future research directions.

\section{Related Work}\label{sec: related_work}
In recent years, many methods have been proposed to learn a latent representation of nodes in an unsupervised manner. Developing a technique for learning network representations inherently contains a plethora of challenges, since a good representation should capture various underlying properties of the network. For instance, many real-world networks consist of tightly connected communities and obey a scale-free property with respect to their degree distribution; in other words, a small numbers of nodes, known as \textit{hubs}, are connected to the majority of nodes. Hence, a structure-preserving method should be able to produce latent representations in which nodes that link to a hub should be  close enough to it in the embeddings' space, while they should also placed far away from each other if they  belong to totally different communities \cite{DBLP:journals/debu/HamiltonYL17}.

The traditional unsupervised feature learning methods aim at factorizing some matrix representation, which has been designed by taking into account  the properties and connections of a given network. MDS \cite{mds}, Laplacian Eigenmaps \cite{laplacian_eigenmap}, Locally Linear Embeddings (LLE) \cite{locally_linear_embedding} and IsoMap \cite{isomap} are just some of those approaches targeting to preserve the first-order proximity of nodes. More recently,  proposed algorithms including GraRep \cite{grarep} and HOPE \cite{hope}, aim at preserving higher order proximities of nodes. Nevertheless, despite the fact that matrix factorization approaches offer an elegant way to capture the desired properties, they mainly suffer from their time complexity.

In recent years,  random walk-based methods \cite{DBLP:journals/debu/HamiltonYL17} have gained considerable attention, mainly due their efficiency. In fact, a very recent study \cite{implicit_factorization} shows that DeepWalk and node2vec \cite{deepwalk, node2vec} implicitly perform matrix factorizations. Following this line of research, distinct random sampling strategies have been proposed and various methods have emerged \cite{random_walk_ddrw, random_walk_struc2vec}.

To the best of our knowledge, very few studies are benefiting from the community structure property of real network to learn node embeddings. The authors of \cite{DBLP:conf/aaai/WangCWP0Y17},  have proposed a matrix factorization-based algorithm that incorporates the community structure into the embedding process, implicitly focusing on the quantity of modularity. The ComE model \cite{come}, proposes a closed-loop procedure among the encoding of communities, learning node embeddings and community detection in the network. As we will present shortly, our work aims at independently learning  node and community (topic) embeddings, and then combining them into expressive topical feature vectors.

\begin{figure}
\includegraphics[scale=0.6]{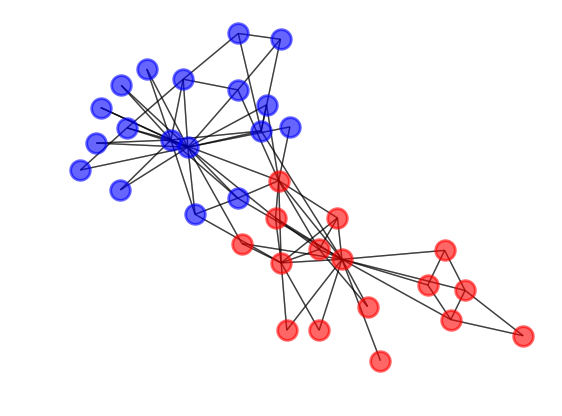}
\caption{The topic assignments in Zachary's karate club network. Each node $v$ is assigned to a community $k$ maximizing $\mathbb{P}(v|k)$ by the \textit{tne-Lda} model. \label{fig:karate}}
\end{figure}

\section{Problem Formulation and Latent Models on Graphs}\label{sec: prob_def}

Let $G=(\mathcal{V}, \mathcal{E})$ be a graph, where $\mathcal{V}$ is the set of nodes and $\mathcal{E}$ denotes the set of edges. Our goal is to find a mapping function $\Phi:\mathcal{V} \rightarrow \mathbb{R}^d$, where $\Phi(v)$ indicates the representation of the vertex $v$ in a lower dimensional space $\mathbb{R}^d$ (which we desire to learn for feeding downstream learning tasks) and $d$ is generally referred to as the embedding or dimension size which is much smaller than the cardinality of the vertex set, $|\mathcal{V}|$.

Node embedding methods based on the popular SkipGram  architecture \cite{word2vec} mainly target to maximize the  log-probability $\max_{\Phi, \widetilde{\Phi}} \sum_{v} \sum_{u \in N_{\gamma}(v)} \log \mathbb{P}(\Phi(u)|\widetilde{\Phi}(v))$,
 where $N_{\gamma}(v)$ denotes the set of reachable nodes by starting from the vertex $v \in \mathcal{V}$ in at most $\gamma$ steps. However, we have to deal with a computational problem when we aim to find $\Phi(u)$ and $\widetilde{\Phi}(v)$ for each $u,v\in\mathcal{V}$, mainly because the computational cost grows significantly as the length $\gamma$ increases due to the sum over $N_{\gamma}(v)$. Therefore, many approaches prefer to approximate the objective function above using random walks as follows:

\begin{equation}\label{eq:random_walk_obj_func}
\mathcal{L}(\Phi, \widetilde{\Phi}) := \max_{\Phi, \widetilde{\Phi}} \sum_{\boldmath{w} \in \mathcal{W}}\sum_{v_i \in \boldmath{w}} \sum_{-\gamma \leq j \leq \gamma} \log \mathbb{P}(\Phi(v_{i+j})|\widetilde{\Phi}(v_i))
\end{equation}

\noindent where $\boldmath{w}=(v_1,\ldots,v_i,\ldots,v_L)$ is a walk of length $L$, $\gamma$ refers to the window size, and $\mathcal{W}$ is a collection of walks. Note that we obtain two different embedding vectors $\Phi(v)$ and $\widetilde{\Phi}(v)$ for each node $v \in \mathcal{V}$, but we will only consider the vector $\Phi(v)$ as a node embedding of $v \in \mathcal{V}$.

Complex networks, such as social or biological networks, consist of latent clusters of different sizes in which the nodes are more likely to be connected to each other \cite{newman_community_structure}. Although some random walk-based methods implicitly benefit from this structural property of networks, our main goal here is to enhance  node embedding vectors using clusters of a given network. We mainly rely on two different approaches to extract latent communities: on random walks and on the network structure itself. For a given graph $G$, we will use the symbol $\mathcal{K}$ to indicate the set of communities of $G$.

\subsection{Random walk-based graph topic models}
Most real-world networks can be expressed as a combination of nested or overlapping communities \cite{overlap_networks}. Therefore, when a random walk is initialized, it does not only visit neighboring nodes, but also traverses communities in the network (see Fig. \ref{fig:complete}). In this regard, we assume that each random walk can be represented as random mixtures over latent communities, and each community can be characterized by a distribution over nodes. In other words, we can write the following generative model for each walk over the network:

\begin{enumerate}
	\item For each $k \in \{1,...,K\}$
	\begin{itemize}
		\item $\phi_k \sim Dir(\beta)$
		\end{itemize}
	\item For each walk $\boldmath{w}=(v_1,...,v_i,...,v_L)$
	\begin{itemize}
		\item $\theta_{w} \sim Dir(\alpha)$
		\item For each vertex $v_i \in \boldsymbol{w}$
		\begin{itemize}
			\item $z_{i} \sim Multinomial(\theta_{w})$
			\item $v_i \sim Multinomial(\phi_{z_{i}})$
		\end{itemize}
	\end{itemize}
\end{enumerate}

\noindent Here, $N$ is the number of walks, $L$ is the length of walks and $K$ is the number of clusters.  



If we consider each random walk as a document and the collection of random walks as a corpus, it can be seen that the statistical process defined above corresponds to the well known Latent Dirichlet Allocation (LDA) model \cite{lda}. Therefore, each community corresponds to a distinct topic in the terminology of NLP (we use the terms \textit{topic} and \textit{community} interchangeably in the rest of the paper).

Now we can use community or topic assignments $z$ of nodes in the walks $w\in \mathcal{W}$ to obtain better vector representations. By replacing a node with its topic label, we aim to predict the nodes in the context of the topic. More formally, we can state our objective function to find community or topic representations as follows:

\begin{align}\label{eq:random_walk_topic_obj_func}
\mathcal{L}(\Psi, &\widetilde{\Psi}) :=\nonumber \\
&\max_{\Psi, \widetilde{\Psi}} \sum_{\boldmath{w} \in \mathcal{W}}\sum_{v_i \in \boldmath{w}} \sum_{-\gamma \leq j \leq \gamma} \log \mathbb{P}(\Psi(v_{i+j})|\widetilde{\Psi}(t_w(v_i))).
\end{align}

\noindent By maximizing the log-probability above, we obtain the embedding vectors $\widetilde{\Psi}(k) \in \mathbb{R}^d$ for each topic label $k \in \mathcal{K}$, which are called as \textit{topic embeddings} or representations. We will refer to this model as \textit{Lda} throughout the paper.

In the previous \textit{Lda} model, the latent community assignment of each node is independently chosen from the topic label of the previous node in the walk. However, the hidden state of the current node can play an important role towards determining the next vertex to visit, as the random walk also traverses through communities. Therefore, we can modify the \textit{Lda} model, and define the following generative process:

\begin{enumerate}
	\item For each $k \in \{1,...,K\}$
	\begin{itemize}
		\item $b_k \sim Dir(b_0)$
		\item $a_{k} \sim Dir(a_0)$
	\end{itemize}
	\item $\pi \sim Dir(p_0)$
	\item For each walk $\boldmath{w}=(v_1,...,v_i,...,v_L)$
	\begin{itemize}
	    \item $z_{1} \sim Dir(\pi)$
		\item For each vertex $v_i \in \boldsymbol{w}$, for all $i<L$
		\begin{itemize}
		    \item $v_{i} \sim Multinomial(b_{z_{i}})$
			\item $z_{i+1} \sim Multinomial(a_{z_{i}})$
		\end{itemize}
		\item $v_{L} \sim Multinomial(b_{z_{L}})$
	\end{itemize}
\end{enumerate}

\noindent The above model is in fact the well-known \textit{Hidden Markov Model} with symmetric Dirichlet priors over transition and emission distributions (we will refer to this model as \textit{Hmm}). Note that, in the generation of each node sequence, the same transition probabilities are used, unlike the topic distribution of the \textit{Lda} model, and the vectors $a_k$ and $b_k$ contain $\mathcal{K}$ and $|\mathcal{V}|$ components, respectively. Moreover, as shown in Lemma \ref{lemma: lda_markov_model_connection}, the \textit{Lda} model can also be viewed as a special case of \textit{Hmm} for the generation of a specific node sequence, after choosing suitable distributions.

\begin{lemma}\label{lemma: lda_markov_model_connection}
The probability of generating the topic and node sequences $\boldsymbol{z}=(z_1,...,z_L)$, $\boldsymbol{w}=(v_1,...,v_L)$ by \textit{Lda} for a given node $\phi_k$ and topic distributions $\theta_{\boldsymbol{w}}$, is equal to the probability of producing the sequences by \textit{Hmm} if the initial, transition and emission probabilities are chosen as $\boldsymbol{\pi} := \theta_{\boldsymbol{w}}$, $\boldsymbol{a_{(\cdot, k)}} := \theta_{\boldsymbol{w},k}$ and  $\boldsymbol{b_{k}} = \phi_{k}$.
\end{lemma}
\begin{proof}
Please see the Appendix.
\end{proof}

\subsection{Network structure-based modeling}
In the previous models, the generated random walks are used to detect the community (or topic) assignment of each node in the given node sequence. Here, we propose two additional model, namely \textit{BigC} and \textit{Louvain}, which directly target to determine communities of nodes from a given network.  The \textit{Louvain} model uses the Louvain method \cite{louvain} to extract communities, while the \textit{BigC} model is based on an overlapping community detection method called BigClam \cite{bigclam}.

\section{Topical Node Embeddings}\label{sec:model}

\begin{figure*}
\centering
	\includegraphics[scale=0.75]{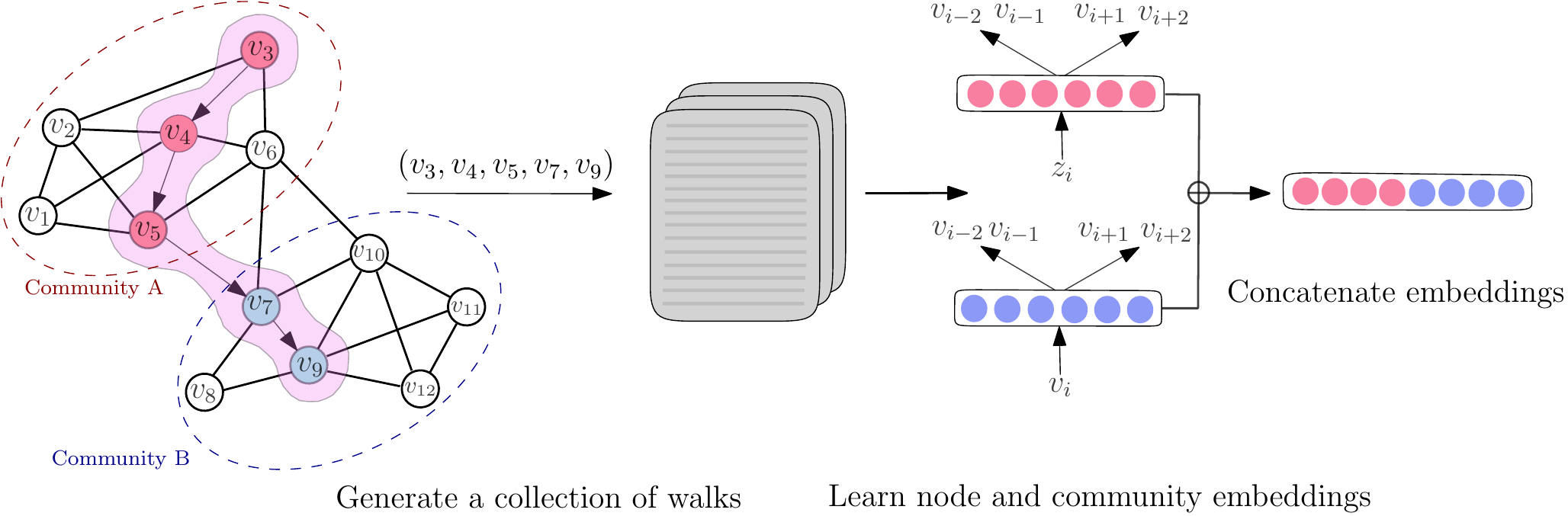}
	\caption{Schematic representation of the TNE model. \label{fig:complete}}
\end{figure*}

In this section, we will describe the proposed Topical Node Embeddings (TNE) model in detail. An overview of the model is given in Fig. \ref{fig:complete}. Our overall goal is to enhance node embedding using information about the underlying topics of the graph.  This can be achieved by learning node and topic embedding vectors independently of each other, jointly maximizing the objectives defined in Equations \eqref{eq:random_walk_obj_func} and \eqref{eq:random_walk_topic_obj_func}. By combining these objective functions, we derive the following equation:

\begin{align*}\label{eq:tne_obj_func}
\max_{\Phi, \widetilde{\Phi}, \Psi, \widetilde{\Psi}} \sum_{\boldmath{w} \in \mathcal{W}}\sum_{v_j \in \boldmath{w}}&\sum_{-\gamma \leq j \leq \gamma}( \log \mathbb{P}(\Phi(v_{i+j})|\widetilde{\Phi}(v_i)) \\
&+ \log \mathbb{P}(\Psi(v_{i+j})|\widetilde{\Psi}(t_w(v_i))) ).
\end{align*}

\noindent In the Skip-Gram model \cite{word2vec}, the probability measure $\mathbb{P}(\cdot|\cdot)$ in the above equation is considered as a softmax function

\begin{align*}
\mathbb{P}(\Phi(u)|\tilde{\Phi}(v)) := \frac{\exp(\Phi(u) \cdot \tilde{\Phi}(v))}{\sum_{u\in V}\exp(\Phi(u) \cdot \tilde{\Phi}(v))},
\end{align*}

\noindent and we adopt the negative sampling technique \cite{word2vec} in order to make our computations more efficient.

After obtaining the node and topic representations, our final step is to efficiently incorporate these feature vectors. For this purpose, we introduce three simple strategies, namely $Max(\cdot)$, $WMean(\cdot)$, and $Min(\cdot)$:

\begin{itemize}
\item $Max(v)$. It produces the final representation for the node $v$ by combining the node and community embeddings: $Max(v) := \Phi(v) \oplus \widetilde{\Psi}(k^*)$. Here, the topic label $k^{*}$ is equal to the parameter $k$ maximizing the expression $\mathbb{P}(v|\widetilde{k}$ and the symbol $\oplus$ denotes the concatenation operation. For instance, if we select the number of topics as $2$ for Zachary's karate club in Figure \ref{fig:karate}, then each node $v$ is assigned to the topic $k$ that has the highest probability.

\item $Min(v)$. The second strategy can be defined as  $Min(v) := \Phi(v) \oplus \widetilde{\Psi}(k^*)$, where $k^{*}=\text{arg}\min_{k}\mathbb{P}(v|k)$.

\item $WMean(v)$. The final strategy is formulated as follows: $WMean(v) :=\Phi(v) \oplus \sum_k \widetilde{\Psi}(k) \cdot \mathbb{P}(v |k)$.
\end{itemize}

\noindent We call the final vector obtained after concatenating the node and topic feature vectors as \textit{topical node embedding}. Algorithm 1 provides the pseudocode of the TNE model. 

The general structure of our framework follows. First, we need a collection of walks over the network to learn node and topic embeddings -- so, any approach such as Deepwalk and Node2vec can be used to perform random walks. Then, we choose a strategy for this collection to get the topic assignment $t_w(v)$ of each node $v \in \mathcal{V}$ in the walk $w \in \mathcal{W}$, based on the latent models on graphs defined in Section \ref{sec: prob_def}. In the first case, we use the stochastic processes \textit{Lda} and \textit{Hmm} described in Section \ref{sec: prob_def}, getting the topical node embedding models of \textit{tne-lda} and \textit{tne-hmm}, respectively. In the second case, the topic assignments are inferred from the network structure based on the  \textit{BigC} and \textit{Louvain} models -- relying on the BigClam and Louvain methods respectively -- and the corresponding topical node embedding models are called \textit{tne-BigC} and \textit{tne-Louvain}.

\vspace{.5cm}
\hspace{-.6cm}
 \begin{tikzpicture}
 \node[draw, rounded corners] {%
\begin{varwidth}{\linewidth}
\begin{algorithmic}
\STATE{\textbf{Algorithm 1:} Topical Node Embeddings}
\REQUIRE Graph $G=(\mathcal{V},\mathcal{E})$, number of walks: $n$, walk length: $\mathcal{L}$, window size: $\gamma$, number of communities: $K$, embedding size: $d$
\ENSURE $|V| \times 2d$ embedding matrix $\Omega$
\STATE{$\mathcal{W}:=\{\boldmath{w}_1,...,\boldmath{w}_N\}$ $\gets$ $\text{GenerateWalks}(G, N, \mathcal{L}, w)$}
\STATE{$t(v) \gets$  $\text{DetectTopics}(\mathcal{W})$}
\STATE{$\mathcal{P}_{\mathcal{W}}$  $\gets$ $\text{GenerateNodeContextPairs}(\mathcal{W}, \gamma)$}
\STATE{$\Phi$ $\gets$ $\text{SkipGram}(\mathcal{P}_{\mathcal{W}}, \gamma, d)$}
\STATE{$\mathcal{P}_{\mathcal{T}}$  $\gets$ $\text{UpdateNodeContextPairs}(\mathcal{P}_{\mathcal{W}}, t)$}
\STATE{$\widetilde{\Psi}$ $\gets$ $\text{SkipGram}(\mathcal{P}_{\mathcal{T}}, \gamma, f)$}
\STATE{$\Omega \gets \text{CombineEmbeddings}(\Phi, \widetilde{\Psi})$ }
\end{algorithmic}
     \end{varwidth}
            };
\end{tikzpicture}

\vspace{.5cm}

\hspace{-.7cm}
\begin{tikzpicture}
 \node[draw, rounded corners] {%
\begin{varwidth}{.98\linewidth}
\begin{algorithmic}
\STATE{\textbf{Algorithm 2:} $\text{GenerateNodeContextPairs}(\mathcal{W}, \gamma)$}
\REQUIRE{A collection of walks $\mathcal{W}:=\{w_1,...,w_N\}$, and window size $\gamma$}
\ENSURE{Node-context pairs $\mathcal{P}$}
\STATE{$\mathcal{P} \gets \{\}$}
\FOR{each walk $w \in \mathcal{W}$}
	\STATE{$\boldmath{p} \gets ()$}
	\FOR{each node $v_i \in \boldmath{w}:=\{v_1,...,v_L\}$}
		\FOR{each $j \in \{\max \{0, i-\gamma\},...,\min\{i+\gamma,n\} \}$}
			\STATE{$\boldmath{p} \gets \text{Append}(\boldmath{p}, (v_i, v_j))$}		
		\ENDFOR
	\ENDFOR
	$\mathcal{P} \gets \mathcal{P} \cup \{\boldmath{p}\}$
\ENDFOR
\end{algorithmic}
     \end{varwidth}
            };
\end{tikzpicture}
\vspace{.2cm}

Afterwards, we produce the node-context pairs to provide the input for the Skip-Gram algorithm, and we learn the latent node representations. By replacing each node $v$ with its topic assignment $t_w(v)$ in the walk $w \in \mathcal{W}$, we obtain a new set of pairs to learn topic embeddings. Finally, we combine the feature vectors depending on our methodology.

\vspace{.5cm}
\hspace{-.75cm}
\begin{tikzpicture}
 \node[draw, rounded corners] {%
\begin{varwidth}{.97\linewidth}
\begin{algorithmic}
\STATE{\textbf{Algorithm 3:} $\text{UpdateNodeContextPairs}(\mathcal{P}_\mathcal{W}, t)$}
\REQUIRE{A collection of node-context sequence pairs \\ $\mathcal{P}_\mathcal{W}:=\{\boldmath{p}_{w_1},...,\boldmath{p}_{w_N}\}$, and topic assignments $t$}
\ENSURE{Node-topic pairs $\mathcal{P}_{\mathcal{T}}$}
\STATE{$\mathcal{P}_{\mathcal{T}} \gets \{\}$}
\FOR{each sequence $\boldmath{p} \in \mathcal{P}_\mathcal{W}$}
\STATE{$\boldmath{p}_{\mathcal{T}} \gets ()$}
\FOR{each pair $(u,v) \in \boldmath{p}$}
	\STATE{$\boldmath{p}_{\mathcal{T}} \gets \text{Append}(\boldmath{p}_{\mathcal{T}}, (u, t(v)))$}
\ENDFOR
\STATE{$\mathcal{P}_{\mathcal{T}} \gets \mathcal{P}_{\mathcal{T}} \cup \{\boldmath{p}_{\mathcal{T}}\}$}
\ENDFOR
\end{algorithmic}
     \end{varwidth}
            };
\end{tikzpicture}

\section{Experiments} \label{sec:experiments}
In this section, we will present the datasets that we use in our experiments and further discuss the performance and effectiveness of the proposed four variations of TNE model  in the tasks of node classification and link prediction. Our model has been implemented in Python and the source code can be found at: \url{https://abdcelikkanat.github.io/projects/TNE/}.

\subsection{Baseline Methods}
We will consider two notable random walk-based approaches and apply our framework to the collection of walks generated by these algorithms.

\begin{itemize}
	\item Deepwalk \cite{deepwalk} uses a very natural sampling strategy in producing walks. At each step, it uniformly chooses a node having connections to the one that it currently resides at, and repeats the same procedure until obtaining a walk of the desired length.  We will refer to this method as \textit{deepwalk-emb}.
	
	\item Node2vec \cite{node2vec} is an extension of Deepwalk, and its walking behavior is controlled by two parameters $p$ and $q$ which provide the ability to discover  distant regions of the network; it also captures  structural similarities between nodes. We will refer to this method as \textit{node2vec-emb}.
\end{itemize} 

\subsection{Parameter Settings}\label{subsec: paramsetting}
In this section, we  describe the parameters' settings that we have used for our experiments and clarify the strategies that we follow. Since both of the random walk sampling strategies that we examine here (\textit{Deepwalk} and \textit{Node2vec}) share many common parameters, we assign all of them to the same typical values. 

More specifically, we consider the number of walks $n=80$, walk length $l=10$, window size $\gamma=10$, and the embedding dimension $d=128$. The return and in-out hyper-parameters $p$, $q$ of Node2vec are simply set to $4.0$ and $1.0$ for all experiments -- so, the walk is encouraged to explore previously unvisited regions of the network. To speed up the training process, we  use negative sampling \cite{word2vec} for all models. We also use stochastic gradient descent (SGD)  \cite{sgd} for optimization, setting the initial learning rate to $0.025$.

For learning the topic assignment of each node in node sequences,  we perform collapsed Gibbs sampling \cite{collapsed_gibbs_sampling} for \textit{tne-Lda} model, and variational message passing \cite{var_msg_pass} for \textit{tne-Hmm}. For all variants of the TNE framework, the number of topics are selected as $K=80$ in the experiments, and $Max$ concatenation method is preferred to obtain final embedding vector.

\subsection{Multi-Label Node Classification}

\begin{figure*}[t]
\centering
	\includegraphics[scale=0.7]{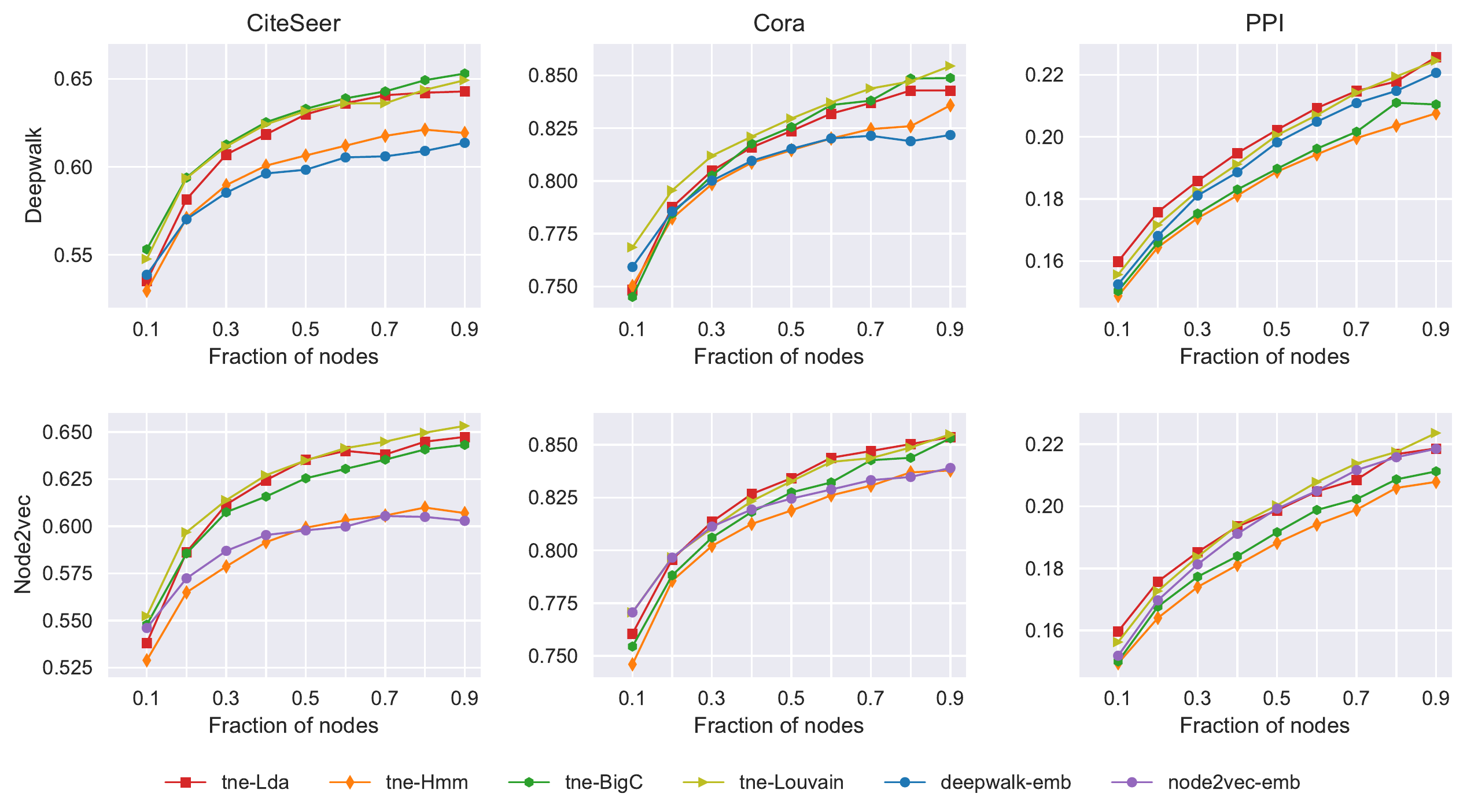}
	\caption{Performance evaluation of the proposed TNE framework against Deepwalk and Node2vec, over a varying fraction of training data. The $x$-axis indicates the ratio of the training dataset, and the $y$-axis shows the Micro-$F_1$ scores for different random walk strategies on three different networks. \label{fig: classification}}
\end{figure*}

In the multi-label node classification experiment, every node of the network is assigned to at least one label; the goal is  to predict the correct node labels by only observing certain fraction of the network.  We use the embedding vectors that we have learned in order to carry out node classification task. We randomly split the collection of feature vectors into training and tests sets, and apply an one-vs-rest logistic regression classifier with $L2$ regularization for optimization. In order to provide more reliable experimental results, we repeat the same procedure for $50$ times.  We use the following three datasets in our experiments.

\begin{itemize}
	\item \textit{CiteSeer} \cite{harp} is a citation network extracted from the CiteSeer library, where nodes represent  research papers and the edges indicate citations between publications.
	
	\item \textit{Protein-Protein Interaction (PPI)} is the subgraph of PPI network for Homo Saphiens and each label corresponds to a biological state \cite{node2vec}.
	
	\item \textit{Cora} \cite{cora} is a citation network consisting of machine learning publications divided into seven categories. Every paper in the corpus is cited or cites at least one other paper.
	
\end{itemize}

\noindent Table \ref{tbl:data-stats} provides the basic statistics of the above datasets.

\begin{table}
\begin{center}
\label{tab:dataset_statistics}
	\begin{tabular}{lccc}
		\toprule
		Name & Citeseer & Cora & PPI \\
		\midrule
		$\#$ Vertices & 3,312 & 2,708 & 3,890 \\
		$\#$ Edges & 4,660 & 5,278 & 38,739 \\
		$\#$ Clusters & 6 & 7 & 50 \\
		\bottomrule
	\end{tabular}
    \end{center}
    \caption{Statistics of the networks used in the multi-label node classification experiment. \label{tbl:data-stats}}
\end{table}

\begin{table}[t]
\begin{center}
\begin{tabular}{rccc}
Name & Citeseer & Cora & PPI \\ \hline
\textit{deepwalk-emb} & 0.554 & 0.808 & 0.174 \\
\textit{tne-Lda} & 0.590 & 0.816 & 0.179 \\
\textbf{Gain/Loss (\%)} & \textbf{6.58} & \textbf{1.04} & \textbf{2.83} \\
\textit{tne-Hmm} & 0.565 & 0.807 & 0.165 \\
\textbf{Gain/Loss (\%)} & \textbf{2.02} & \textbf{-0.03} & \textbf{-5.01} \\
\textit{tne-BigC} & 0.591 & 0.814 & 0.168 \\
\textbf{Gain/Loss (\%)} & \textbf{6.69} & \textbf{0.81} & \textbf{-3.14} \\
\textit{tne-Louvain} & 0.589 & 0.819 & 0.175 \\
\textbf{Gain/Loss (\%)} & \textbf{6.45} & \textbf{1.42} & \textbf{0.80} \\
\textit{} &  &  &  \\
Name & Citeseer & Cora & PPI \\ \hline
\textit{node2vec-emb} & 0.551 & 0.814 & 0.174 \\
\textit{tne-Lda} & 0.591 & 0.822 & 0.175 \\
\textbf{Gain/Loss (\%)} & \textbf{7.32} & \textbf{0.96} & \textbf{0.47} \\
\textit{tne-Hmm} & 0.556 & 0.807 & 0.164 \\
\textbf{Gain/Loss (\%)} & \textbf{0.84} & \textbf{-0.93} & \textbf{-5.68} \\
\textit{tne-BigC} & 0.586 & 0.817 & 0.169 \\
\textbf{Gain/Loss (\%)} & \textbf{6.31} & \textbf{0.28} & \textbf{-2.90} \\
\textit{tne-Louvain} & 0.593 & 0.823 & 0.173 \\
\textbf{Gain/Loss (\%)} & \textbf{7.58} & \textbf{1.10} & \textbf{-0.47}
\end{tabular}
\end{center}
\caption{Macro-$F_1$ scores for multi-label node classification, where 50\% of the nodes are used for training. The top table shows the performance of the various TNE models applied on walks extracted by Deepwalk, as well as the performance of the Deepwalk algorithm. Similarly, the bottom table gives the performance of TNE with respect to Node2vec. \label{tab:classification_50percent}}
\end{table}

\subsubsection*{Experiment results}
Figure \ref{fig: classification} depicts the Micro-$F_1$ scores for the variants of the TNE framework as well as for the baseline methods,  with respect to the number of nodes in the training set. In Table \ref{tab:classification_50percent}, the Macro-$F_1$ scores are shown for the case where the size of training and test sets are equal. As it can be seen,  \textit{tne-BigC} provides a gain of up to $6.69\%$  compared to the raw Deepwalk model (\textit{deeepwalk-emb}), and up to $6.31\%$ compared to Node2vec (\textit{node2vec-emb}) on the \textit{Citeseer} dataset. 

Although the general performance of the two feature learning methods Node2vec and Deepwalk are the same over the PPI network, \textit{tne-Lda} model increases the score up to $2.83\%$ while \textit{tne-Louvain} cannot show a great performance as much as it.

\subsection{The effect of the number of topics}
In this paragraph, we analyze the effect of the number of topics (or clusters)  in the performance of our framework. We perform experiments on the CiteSeer network and we examine the \textit{tne-Lda} and \textit{tne-Hmm} models on the collection of random walks generated by Deepwalk and Node2vec.  All the parameter settings are  the same as those described in Subsection \ref{subsec: paramsetting}, except the number of topics. Figure \ref{fig: number_of_topics} indicates that the increase in the number of topics makes positive contribution up to a certain value for \textit{tne-Lda} model. On the other hand, this is not valid for \textit{tne-Hmm}; it performs better for $K=120$ over  both random walk strategies. The chosen number of topics shows its importance for large training data sizes -- the scores get closer to each other when the training size decreases.

\begin{figure}
\centering
	\includegraphics[scale=0.6]{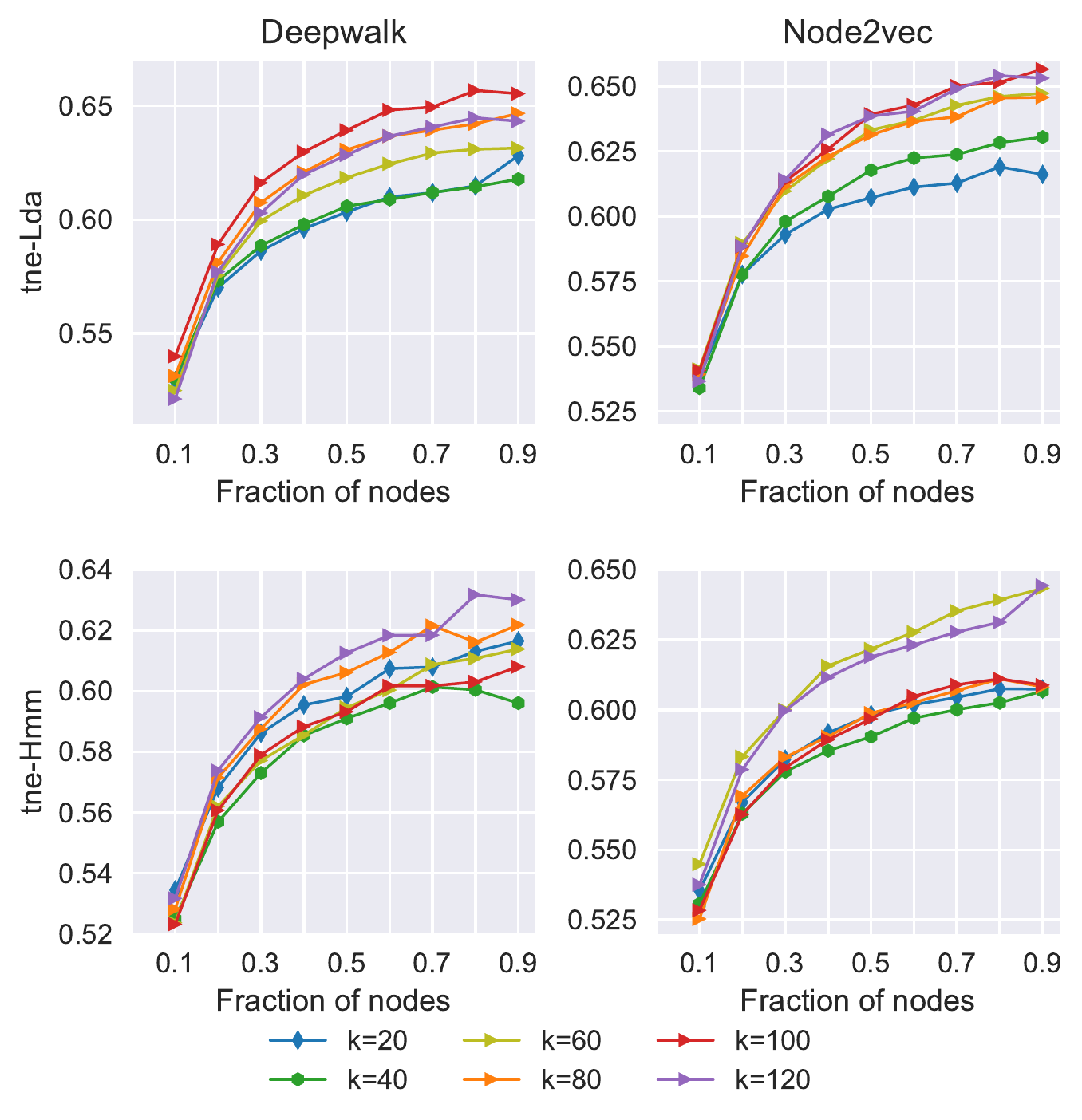}
	\caption{Micro-$F_1$ scores for various values of the number of topics for the CiteSeer network. \label{fig: number_of_topics}}
\end{figure}

\subsection{The effect of the concatenation strategy}
In Section \ref{sec:model}, we have described how to combine the node and topic feature vectors, in order to construct topical node embeddings. Here, we perform several experiments to observe the behavior of those strategies over varying training data sizes.  Figure \ref{fig:concatenation_method} depicts the Micro-$F_1$ scores on the CiteSeer network. As it can be seen, the $MAX$ and $WMean$ strategies highly outperform  the third one across all cases, and their scores are highly close to each other.

\begin{figure}[t]
\centering
	\includegraphics[scale=0.6]{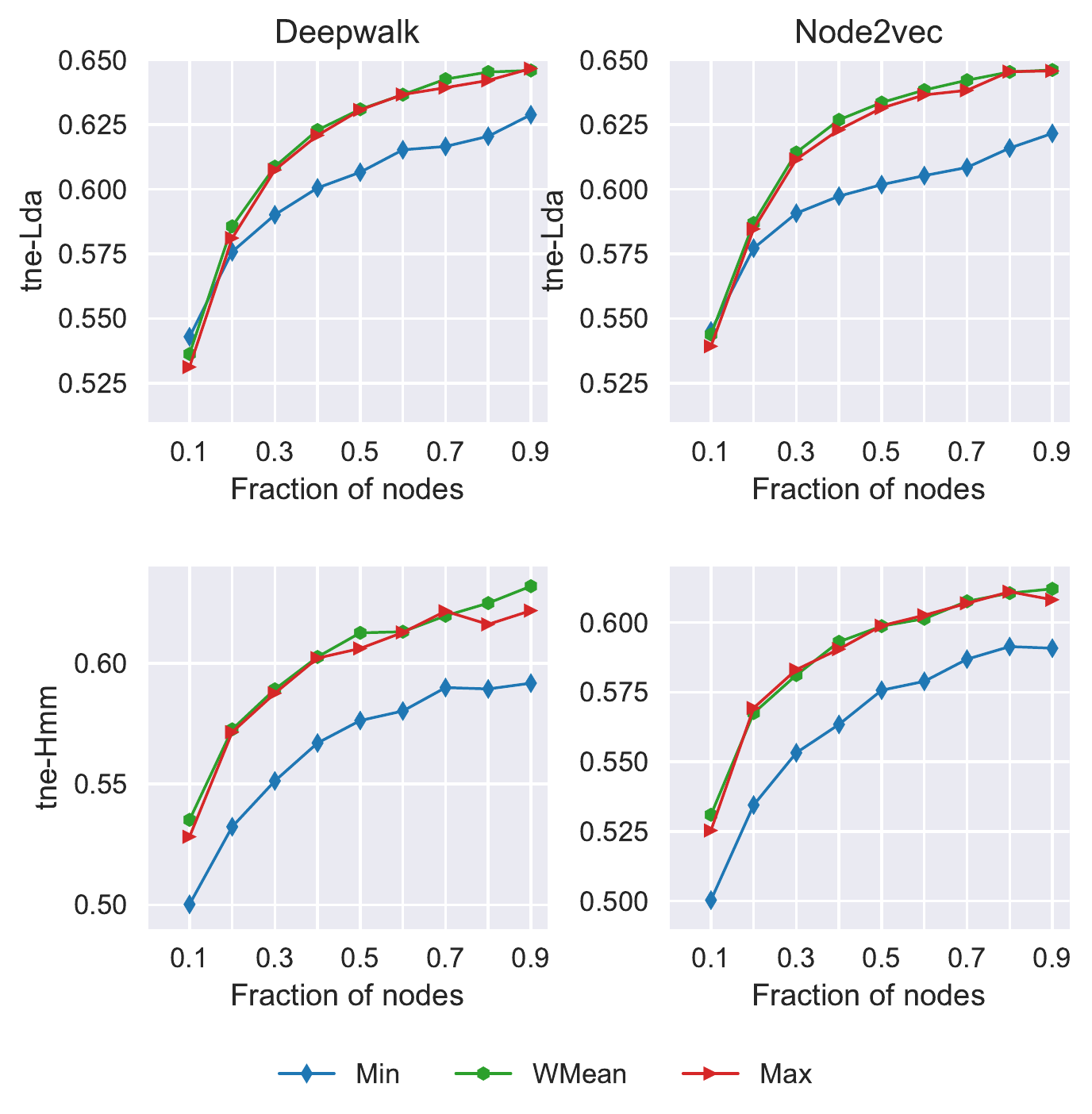}
	\caption{The effect of different embedding concatenation methods for the CiteSeer network. \label{fig:concatenation_method}}
\end{figure}

\begin{table*}[t]
\centering
\begin{tabular}{llcccccccc}
 &  & \multicolumn{2}{c}{(a)} & \multicolumn{2}{c}{(b)} & \multicolumn{2}{c}{(c)} & \multicolumn{2}{c}{(d)} \\
 & \multicolumn{1}{l|}{} & Deepwalk & \multicolumn{1}{c|}{Node2vec} & Deepwalk & \multicolumn{1}{c|}{Node2vec} & Deepwalk & \multicolumn{1}{c|}{Node2vec} & Deepwalk & Node2vec \\ \cline{2-10} 
\multirow{5}{*}{\rotatebox{90}{Gnutella}} & \multicolumn{1}{l|}{} & 0.5952 & \multicolumn{1}{c|}{0.5944} & 0.7050 & \multicolumn{1}{c|}{0.7148} & 0.5825 & \multicolumn{1}{c|}{0.6194} & 0.5790 & 0.6171 \\
 & \multicolumn{1}{r|}{\textit{tne-Lda}} & 0.5920 & \multicolumn{1}{c|}{0.5991} & 0.7043 & \multicolumn{1}{c|}{0.7086} & 0.5852 & \multicolumn{1}{c|}{0.6224} & 0.5820 & 0.6208 \\
 & \multicolumn{1}{r|}{\textit{tne-Hmm}} & 0.5961 & \multicolumn{1}{c|}{0.5916} & \textbf{0.7125} & \multicolumn{1}{c|}{\textbf{0.7261}} & 0.5821 & \multicolumn{1}{c|}{0.6179} & 0.5732 & 0.6137 \\
 & \multicolumn{1}{r|}{\textit{tne-BigC}} & 0.5988 & \multicolumn{1}{c|}{\textbf{0.6017}} & 0.7047 & \multicolumn{1}{c|}{0.7227} & 0.5863 & \multicolumn{1}{c|}{\textbf{0.6272}} & 0.5804 & \textbf{0.6256} \\
 & \multicolumn{1}{r|}{\textit{tne-Louvain}} & \textbf{0.5998} & \multicolumn{1}{c|}{0.5945} & 0.6991 & \multicolumn{1}{c|}{0.7071} & \textbf{0.5873} & \multicolumn{1}{c|}{0.6188} & \textbf{0.5827} & 0.6158 \\ \cline{2-10} 
\multirow{5}{*}{\rotatebox{90}{Facebook}} & \multicolumn{1}{l|}{} & 0.9862 & \multicolumn{1}{c|}{0.9865} & 0.7537 & \multicolumn{1}{c|}{0.7505} & 0.9839 & \multicolumn{1}{c|}{0.9831} & 0.9840 & 0.9834 \\
 & \multicolumn{1}{l|}{\textit{tne-Lda}} & 0.9882 & \multicolumn{1}{c|}{0.9888} & 0.7772 & \multicolumn{1}{c|}{0.7749} & 0.9859 & \multicolumn{1}{c|}{0.9861} & 0.9861 & 0.9866 \\
 & \multicolumn{1}{l|}{\textit{tne-Hmm}} & \textbf{0.9884} & \multicolumn{1}{c|}{0.9884} & \textbf{0.7789} & \multicolumn{1}{c|}{\textbf{0.7784}} & 0.9864 & \multicolumn{1}{c|}{0.9860} & 0.9868 & 0.9862 \\
 & \multicolumn{1}{l|}{\textit{tne-BigC}} & 0.9882 & \multicolumn{1}{c|}{\textbf{0.9890}} & 0.7715 & \multicolumn{1}{c|}{0.7731} & \textbf{0.9869} & \multicolumn{1}{c|}{\textbf{0.9864}} & \textbf{0.9870} & \textbf{0.9867} \\
 & \multicolumn{1}{l|}{\textit{tne-Louvain}} & 0.9881 & \multicolumn{1}{c|}{0.9888} & 0.7597 & \multicolumn{1}{c|}{0.7615} & 0.9846 & \multicolumn{1}{c|}{0.9842} & 0.9847 & 0.9845 \\ \cline{2-10} 
\multirow{5}{*}{\rotatebox{90}{arXiv}} & \multicolumn{1}{l|}{} & 0.9262 & \multicolumn{1}{c|}{0.9314} & 0.7256 & \multicolumn{1}{c|}{0.7254} & 0.9249 & \multicolumn{1}{c|}{0.9304} & 0.9253 & 0.9312 \\
 & \multicolumn{1}{l|}{\textit{tne-Lda}} & \textbf{0.9328} & \multicolumn{1}{c|}{0.9346} & 0.7232 & \multicolumn{1}{c|}{0.7249} & \textbf{0.9335} & \multicolumn{1}{c|}{\textbf{0.9319}} & \textbf{0.9337} & 0.9323 \\
 & \multicolumn{1}{l|}{\textit{tne-Hmm}} & 0.9220 & \multicolumn{1}{c|}{0.9332} & 0.7223 & \multicolumn{1}{c|}{0.7290} & 0.9207 & \multicolumn{1}{c|}{0.9304} & 0.9212 & 0.9321 \\
 & \multicolumn{1}{l|}{\textit{tne-BigC}} & 0.9271 & \multicolumn{1}{c|}{0.9309} & \textbf{0.7273} & \multicolumn{1}{c|}{0.7311} & 0.9237 & \multicolumn{1}{c|}{0.9288} & 0.9228 & 0.9294 \\
 & \multicolumn{1}{l|}{\textit{tne-Louvain}} & 0.9302 & \multicolumn{1}{c|}{\textbf{0.9353}} & 0.7320 & \multicolumn{1}{c|}{\textbf{0.7375}} & 0.9274 & \multicolumn{1}{c|}{0.9340} & 0.9266 & \textbf{0.9342}
\end{tabular}
\caption{Area Under Curve (AUC) scores for the link prediction task with four different binary operators: (a) Hadamard, (b) Average, (c) Weighted-L1, and (d) Weighted-L2. The first row of each block corresponds to the performance of \textit{deepwalk-emb} and \textit{node2vec-emb}. \label{tab:link_prediction}}
\end{table*}

\subsection{Link Prediction}
In the link prediction task, we have a limited access to the edges of the network, and our goal is to predict the missing (unseen) edges between nodes. We divide the edge set  of a given network into two parts to form training and test sets, by randomly removing $50\%$ of  the edges (the network  remains connected during the process). The removed edges are later used as positive samples in the test set. The same number of node pairs that does not exist in the initial network is chosen to obtain negative samples for each training and test sets. The node embedding vectors are converted into edge features based on the binary operators listed in Table \ref{tab:binary_operators}. 

We perform all experiments  using the logistic regression classifier with $L2$ regularization on the following networks:

\begin{itemize}
	\item \textit{Gnutella} \cite{gnutella} is the peer-to-peer file sharing network collected in August 9, 2012. It consists of $8,114$ nodes and $26,013$ edges. 
	
	\item \textit{Facebook} \cite{facebook} is a social network containing $4,039$ nodes and $88,234$ edges. 
	
	\item \textit{arXiv GR-QC} \cite{grqc} is a co-authorship network consisting of $5,242$ nodes and $14,496$ edges. 
\end{itemize}


\subsubsection*{Experiment results} 
Table \ref{tab:link_prediction} presents the area under curve (AUC) scores for the link prediction task. As it can be seen, the proposed TNE framework  outperforms the baseline methods in all cases. For the \textit{Facebook} network, \textit{tne-BigC} gives the best results for all but the average operator -- which also corresponds to the best performing model across all different settings. 


\begin{table}
\begin{center}
	\begin{tabular}{rc}
	    \toprule
		Operator & Definition \\
		\midrule
		Hadamard & $v \circ u$ \\
		Average & $0.5\cdot(v+u)$ \\
		Weighted-L1 & $|v-u|_1$ \\
		Weighted-L2 & $|v-u|_2 $\\
		\bottomrule
	\end{tabular}
    \end{center}
\caption{Binary operators for learning edge feature vectors from node embeddings. \label{tab:binary_operators}}    
\end{table}

\section{Conclusions and Future Work} \label{sec:conclusions}
In this paper, we have proposed TNE, a latent model for representation learning on networks. TNE takes advantage of the  topics (or clusters) that a node belongs to -- leading to the concept of topical node embeddings. That way, TNE is capable of producing enriched latent node representations, compared to traditional random walk-based approaches,  leading to improved performance results in the tasks of node classification and link prediction.

\par Currently, TNE can be applied along with random walk-based approaches. An interesting future direction is how to extend the framework to include other NRL algorithms.  Moreover, motivated by the hierarchical community structure that many real networks follow, an interesting future direction would be to extend the framework towards learning hierarchical node embeddings.  Lastly, we plan to evaluate TNE in the task of community detection.

\bibliography{main}

\begin{thebibliography}{10}

\bibitem{laplacian_eigenmap}
{\sc M.~Belkin and P.~Niyogi}, {\em Laplacian eigenmaps and spectral techniques
  for embedding and clustering}, in NIPS, 2002, pp.~585--591.

\bibitem{louvain}
{\sc V.~D. Blondel, J.-L. Guillaume, R.~Lambiotte, and E.~Lefebvre}, {\em Fast
  unfolding of communities in large networks}, J. Stat. Mech., 2008 (2008),
  p.~P10008.

\bibitem{sgd}
{\sc L.~Bottou}, {\em Stochastic gradient learning in neural networks}, in In
  Proceedings of Neuro-Nîmes. EC2, 1991.

\bibitem{grarep}
{\sc S.~Cao, W.~Lu, and Q.~Xu}, {\em Grarep: Learning graph representations
  with global structural information}, in CIKM, 2015, pp.~891--900.

\bibitem{come}
{\sc S.~Cavallari, V.~W. Zheng, H.~Cai, K.~C.-C. Chang, and E.~Cambria}, {\em
  Learning community embedding with community detection and node embedding on
  graphs}, in CIKM, 2017, pp.~377--386.

\bibitem{harp}
{\sc H.~Chen, B.~Perozzi, Y.~Hu, and S.~Skiena}, {\em Harp: Hierarchical
  representation learning for networks}, in AAAI, 2018.

\bibitem{lda}
{\sc M.~I.~J. David M.~Blei, Andrew Y.~Ng}, {\em Latent dirichlet allocation},
  Journal of Machine Learning Research., 3 (2003), pp.~993--1022.

\bibitem{newman_community_structure}
{\sc M.~Girvan and M.~E.~J. Newman}, {\em Community structure in social and
  biological networks}, PNAS, 99 (2002), pp.~7821--7826.

\bibitem{collapsed_gibbs_sampling}
{\sc T.~L. Griffiths and M.~Steyvers}, {\em Finding scientific topics}, PNAS,
  101 (2004), pp.~5228--5235.

\bibitem{node2vec}
{\sc A.~Grover and J.~Leskovec}, {\em Node2vec: Scalable feature learning for
  networks}, in KDD, 2016, pp.~855--864.

\bibitem{DBLP:journals/debu/HamiltonYL17}
{\sc W.~L. Hamilton, R.~Ying, and J.~Leskovec}, {\em Representation learning on
  graphs: Methods and applications}, {IEEE} Data Eng. Bull., 40 (2017),
  pp.~52--74.

\bibitem{mds}
{\sc T.~Hofmann and J.~Buhmann}, {\em Multidimensional scaling and data
  clustering},  (1995), pp.~459--466.

\bibitem{grqc}
{\sc J.~Leskovec, J.~Kleinberg, and C.~Faloutsos}, {\em Graph evolution:
  Densification and shrinking diameters}, ACM Trans. Knowl. Discov. Data, 1
  (2007).

\bibitem{facebook}
{\sc J.~Leskovec and J.~J. Mcauley}, {\em Learning to discover social circles
  in ego networks}, in NIPS, 2012, pp.~539--547.

\bibitem{random_walk_ddrw}
{\sc J.~Li, J.~Zhu, and B.~Zhang}, {\em Discriminative deep random walk for
  network classification}, in ACL, 2016, pp.~1004--1013.

\bibitem{topical_word}
{\sc Y.~Liu, Z.~Liu, T.-S. Chua, and M.~Sun}, {\em Topical word embeddings}, in
  AAAI, 2015, pp.~2418--2424.

\bibitem{word2vec}
{\sc T.~Mikolov, I.~Sutskever, K.~Chen, G.~Corrado, and J.~Dean}, {\em
  Distributed representations of words and phrases and their compositionality},
  in NIPS, 2013, pp.~3111--3119.

\bibitem{hope}
{\sc M.~Ou, P.~Cui, J.~Pei, Z.~Zhang, and W.~Zhu}, {\em Asymmetric transitivity
  preserving graph embedding}, in KDD, 2016, pp.~1105--1114.

\bibitem{overlap_networks}
{\sc G.~Palla, I.~Der{\'e}nyi, I.~Farkas, and T.~Vicsek}, {\em Uncovering the
  overlapping community structure of complex networks in nature and society},
  Nature, 435 (2005), p.~814.

\bibitem{deepwalk}
{\sc B.~Perozzi, R.~Al-Rfou, and S.~Skiena}, {\em Deepwalk: Online learning of
  social representations}, in KDD, 2014, pp.~701--710.

\bibitem{implicit_factorization}
{\sc J.~Qiu, Y.~Dong, H.~Ma, J.~Li, K.~Wang, and J.~Tang}, {\em Network
  embedding as matrix factorization: Unifying deepwalk, line, pte, and
  node2vec}, in WSDM, 2018, pp.~459--467.

\bibitem{random_walk_struc2vec}
{\sc L.~F. Ribeiro, P.~H. Saverese, and D.~R. Figueiredo}, {\em Struc2vec:
  Learning node representations from structural identity}, in KDD,
  pp.~385--394.

\bibitem{gnutella}
{\sc M.~Ripeanu, A.~Iamnitchi, and I.~Foster}, {\em Mapping the gnutella
  network}, IEEE Internet Computing, 6 (2002), pp.~50--57.

\bibitem{locally_linear_embedding}
{\sc S.~T. Roweis and L.~K. Saul}, {\em Nonlinear dimensionality reduction by
  locally linear embedding}, Science, 290 (2000), pp.~2323--2326.

\bibitem{cora}
{\sc P.~Sen, G.~Namata, M.~Bilgic, L.~Getoor, B.~Gallagher, and
  T.~Eliassi-Rad}, {\em Collective classification in network data},  (2008).

\bibitem{isomap}
{\sc J.~B. Tenenbaum, V.~d. Silva, and J.~C. Langford}, {\em A global geometric
  framework for nonlinear dimensionality reduction}, Science, 290 (2000),
  pp.~2319--2323.

\bibitem{DBLP:conf/aaai/WangCWP0Y17}
{\sc X.~Wang, P.~Cui, J.~Wang, J.~Pei, W.~Zhu, and S.~Yang}, {\em Community
  preserving network embedding}, in AAAI, 2017, pp.~203--209.

\bibitem{var_msg_pass}
{\sc J.~Winn and C.~M. Bishop}, {\em Variational message passing}, J. Mach.
  Learn. Res., 6 (2005), pp.~661--694.

\bibitem{bigclam}
{\sc J.~Yang and J.~Leskovec}, {\em Overlapping community detection at scale: A
  nonnegative matrix factorization approach}, in WSDM, 2013, pp.~587--596.

\end{thebibliography}
\bibliographystyle{siam}

\section*{Appendix}

\subsubsection*{Proof of Lemma \ref{lemma: lda_markov_model_connection}}
Let $\boldsymbol{w}$ and $\boldsymbol{z}$ be the node and topic sequences that are generated by the Markov model defined in Section \ref{sec: prob_def} for given parameters $\boldsymbol{\pi}$, $\boldsymbol{a}$ and $\boldsymbol{b}$ with probability

\begin{equation}\label{eq:vc_walk_prob}
\mathbb{P}(\boldsymbol{w}, \boldsymbol{z}|\boldsymbol{a}, \boldsymbol{b}, \boldsymbol{\pi}) = \pi_{z_1}\left(\prod_{l=1}^{L-1}b_{z_{l},v_{l}}a_{z_{l},z_{l+1}}\right)b_{z_L,v_L}.
\end{equation}

\noindent The probability of generating the same pairs $(\boldsymbol{w}$ and $\boldsymbol{z})$ by the \textit{Lda} model is
\begin{equation}\label{eq: lda_walk_prob}
\mathbb{P}(\boldsymbol{w}, \boldsymbol{z}|\phi_k, \theta_{\boldsymbol{w}}) =  \prod_{l=1}^{L}\phi_{ z_l,v_l}\theta_{\boldsymbol{w},z_l},
\end{equation}
for a given $\phi_k \sim Dir(\beta)$ and $\theta_{\boldsymbol{w}} \sim Dir(\alpha)$, where $\alpha$ and $\beta$ are the hyper-parameters. If the emission, transition and initial state probabilities of the Markov chain are chosen as follows: $\boldsymbol{b}_{k} = \phi_{k}$, $\boldsymbol{a}_{(\cdot, k)} := \theta_{\boldsymbol{w}, k}$ and $\boldsymbol{\pi} = \boldsymbol{\theta}_{\boldsymbol{w}}$, then Eq.  \eqref{eq:vc_walk_prob} can be re-written as

\begin{align*}\label{eq:vc_walk_prob}
\mathbb{P}(\boldsymbol{w}, \boldsymbol{z}|\boldsymbol{a}, \boldsymbol{b}, \boldsymbol{\pi}) &= \pi_{z_1}\left(\prod_{l=1}^{L-1}b_{z_{l},v_{l}}a_{z_{l},z_{l+1}}\right)b_{z_L,v_L} \\
&= \boldsymbol{\theta}_{\boldsymbol{w},z_1}\left(\prod_{l=1}^{L-1}\phi_{z_{l},v_{l}}\boldsymbol{\theta}_{\boldsymbol{w},z_{l+1}}\right)\phi_{z_L,v_L} \\
&= \prod_{l=1}^{L}\phi_{z_{l},v_{l}}\boldsymbol{\theta}_{\boldsymbol{w},z_l},
\end{align*}

which is equal to the probability given in Eq. \eqref{eq: lda_walk_prob}.

\end{document}